\documentclass[11pt]{article}

\usepackage{fullpage,times}

\usepackage{amsthm,amsfonts,amsmath,amssymb,epsfig,color,float,graphicx,verbatim}
\usepackage{algorithm,algorithmic} \usepackage{hyperref}

\newtheorem{theorem}{Theorem} \newtheorem{proposition}{Proposition}
\newtheorem{lemma}{Lemma} 
\newtheorem{definition}{Definition} \newtheorem{remark}{Remark}

\newcommand{\reals}{\mathbb{R}}

 \newcommand{\bx}{\mathbf{x}}
\newcommand{\bw}{\mathbf{w}} \newcommand{\bg}{\mathbf{g}}
 
 \newcommand{\bp}{\mathbf{p}}
\newcommand{\bv}{\mathbf{v}} \newcommand{\bz}{\mathbf{z}}
\newcommand{\bc}{\mathbf{c}} 
\newcommand{\bh}{\mathbf{h}} \newcommand{\by}{\mathbf{y}}
\newcommand{\bk}{\mathbf{k}} \newcommand{\bn}{\mathbf{n}}

\newcommand{\balpha}{\boldsymbol{\alpha}}

 \newcommand{\Ocal}{\mathcal{O}}

\newcommand{\norm}[1]{\|#1\|} \newcommand{\inner}[1]{\langle#1\rangle}

\newcommand{\ignore}[1]{}

\newcommand{\secref}[1]{Sec.~\ref{#1}}
\newcommand{\subsecref}[1]{Subsection~\ref{#1}}

\renewcommand{\eqref}[1]{Eq.~(\ref{#1})}
\newcommand{\lemref}[1]{Lemma~\ref{#1}}

\newcommand{\thmref}[1]{Thm.~\ref{#1}}

\newenvironment{Ouralgorithm}[1][\  ] %
{
\rm
\begin{tabbing}
.\=...\=...\=...\=...\=  \+ \kill
} %
{\end{tabbing}
}
\floatstyle{ruled}
\newfloat{Algorithm}{thp}{alg}
\newenvironment{Balgorithm} %
{
\begin{minipage}{1.0\linewidth}
\begin{Ouralgorithm} %
} { \end{Ouralgorithm} \end{minipage} }
\title{An Algorithm for Training Polynomial Networks}
\author{Roi Livni\\ The Hebrew University \\ roi.livni@mail.huji.ac.il
  \and Shai Shalev-Shwartz\\ The Hebrew University \\
  shais@cs.huji.ac.il \and
  Ohad Shamir\\
  Weizmann Institute of Science\\
  ohad.shamir@weizmann.ac.il } \date{}

\begin{document}

\maketitle

\begin{abstract}
We consider deep neural networks, in which the output of each node is a quadratic function of
its inputs. Similar to other deep architectures, these networks can
compactly represent any function on a finite training set. The main goal of
this paper is the derivation of an efficient layer-by-layer algorithm for
training such networks, which we denote as the \emph{Basis Learner}. The
algorithm is a universal learner in the sense that the training error is
guaranteed to decrease at every iteration, and can eventually reach zero
under mild conditions. We present practical implementations of this
algorithm, as well as preliminary experimental results. We also compare our
deep architecture to other shallow architectures for learning polynomials,
in particular kernel learning.
\end{abstract}

\section{Introduction}

One of the most significant recent developments in machine learning has been
the resurgence of ``deep learning'', usually in the form of artificial neural
networks. These systems are based on a multi-layered architecture, where the
input goes through several transformations, with higher-level concepts
derived from lower-level ones. Thus, these systems are considered to be
particularly suitable for hard AI tasks, such as computer vision and language
processing.

The history of such multi-layered systems is long and uneven. They have been
extensively studied in the 80's and early 90's, but with mixed success, and
were eventually displaced to a large extent by shallow architectures such as
the Support Vector Machine (SVM) and boosting algorithms. These shallow
architectures not only worked well in practice, but also came with provably
correct and computationally efficient training algorithms, requiring tuning
of only a small number of parameters - thus allowing them to be incorporated
into standard software packages.

However, in recent years, a combination of algorithmic advancements, as well
as increasing computational power and data size, has led to a breakthrough in
the effectiveness of neural networks, and deep learning systems have shown
very impressive practical performance on a variety of domains (a few examples
include
\cite{lecun1995convolutional,HiOsTe06,ranzato2007unsupervised,bengio2007scaling,ColWes08,lee2009convolutional,LRMDCCDN12}
as well as \cite{Beng09} and references therein). This has led to a
resurgence of interest in such learning systems.

Nevertheless, a major caveat of deep learning is - and always has been - its
strong reliance on heuristic methods. Despite decades of research, there is
no clear-cut guidance on how one should choose the architecture and size of
the network, or the type of computations it performs. Even when these are
chosen, training these networks involves non-convex optimization problems,
which are often quite difficult. No worst-case guarantees are possible, and
pulling it off successfully is still much of a black art, requiring
specialized expertise and much manual work.

In this note, we propose an efficient algorithm to build and train a deep
network for supervised learning, with some formal guarantees. The algorithm
has the following properties:
\begin{itemize}
\item It constructs a deep architecture, one which relies on its
    multi-layered structure in order to compactly represent complex
    predictors.
\item It provably runs in polynomial time, and is amenable to theoretical
    analysis and study. Moreover, the algorithm does not rely on
    complicated heuristics, and is easy to implement.
\item The algorithm is a universal learner, in the sense that the training
    error is guaranteed to decrease as the network increases in size,
    ultimately reaching zero under mild conditions. %  Moreover, a
%single run of the algorithm constructs an entire
%  curve of predictors trading-off bias and variance (or in other
%  words, training error vs. generalization performance).
  % \item All the layers except the last one are built one-by-one, and
  %   in a purely unsupervised manner, without using the labels. Thus,
  %   our algorithm implements the principle of unsupervised
  %   layer-by-layer training of the network, creating
  %   higher-and-higher level representations of the data, followed by
  %   a final supervised learning phase. This is particularly helpful
  %   when training multiclass or multitask predictors.
    \item In its basic idealized form, the algorithm is parameter-free. The
        network is grown incrementally, where each added layer decreases
        the bias while increasing the variance. The process can be stopped
        once satisfactory performance is obtained. The architectural
        details of the network are automatically determined by theory. We
        describe a more efficient variant of the algorithm, which requires
        specifying the maximal width of the network in advance. Optionally,
        one can do additional fine-tuning (as we describe later on), but
        our experimental results indicate that even this rough tuning is
        already sufficient to get promising results.
\end{itemize}

The algorithm we present trains a particular type of deep learning system,
where each computational node computes a linear or quadratic function of its
inputs. Thus, the predictors we learn are polynomial functions over the input
space (which we take here to be $\reals^d$). The networks we learn are also
related to \emph{sum-product networks}, which have been introduced in the
context of efficient representations of partition functions
\cite{PoDo11,DeBe11}.

The derivation of our algorithm is inspired by ideas from \cite{VCA}, used
there for a different purpose. At its core, our method attempts to build a
network which provides a good approximate basis for the values attained by
all polynomials of bounded degree over the training instances.
% Crucially, it utilizes the deep structure of multi-layered networks in order to provide a
%\emph{compact} representation. This is important, since the dimension of the
%vector space of all polynomials of bounded degree in $\reals^d$ is huge
%(exponential in the degree), and there is no hope to build such a
%representation explicitly. Efficiently representing high-dimensional spaces
%can also be achieved using kernels, but our approach is very different and
%has some important advantages which we discuss later on.
Similar to a well-known principle in modern deep learning, the layers of our
network are built one-by-one, creating higher-and-higher level
representations of the data. Once such a representation is built, a final
output layer is constructed by solving a simple convex optimization problem.

The rest of the paper is structured as follows. In
\secref{sec:preliminaries}, we introduce notation. The heart of our paper is
\secref{sec:algorithmanalysis}, where we present our algorithm and analyze
its properties. In \secref{sec:samplecomplexity}, we discuss sample
complexity (generalization) issues. In \secref{sec:kernels} we compare our
deep architecture for learning polynomials to the shallow architecture
obtained by kernel learning. In \secref{sec:experiments}, we present
preliminary experimental results.

%Overall, we believe our work opens the door to the derivation of principled
%deep learning algorithms, with many possible avenues for future work.

\section{Preliminaries}\label{sec:preliminaries}

We use bold-face letters to denote vectors. In particular, $\mathbf{1}$
denotes the all-ones vector. For any two vectors $\bg=(g_1,\ldots,g_d)$,
$\bh=(h_1,\ldots,h_d)$, we let $\bg\circ\bh$ denote their Hadamard product,
namely the vector $(g_1h_1,\ldots,g_dh_d)$. $\norm{\cdot}$ refers to the
Euclidean norm. $\text{Ind}(\cdot)$ refers to the indicator function.

For two matrices $F,G$ with the same number of rows, we let $[F~~G]$ denote
the new matrix formed by concatenating the columns of $F,G$. For a matrix
$F$, $F_{i,j}$ refers to the entry in row $i$ and column $j$; $F_j$ refers to
its $j$-th column; and $|F|$ refers to the number of columns.

We assume we are given a labeled training data
$\{(\bx_1,y_1),\ldots,(\bx_m,y_m)\}$, where each $\bx_i$ is in $\reals^d$,
and $y_i$ is a scalar label/target value. We let $X$ denote the matrix such
that $X_{i,j}=x_{i,j}$, and $\by$ is the vector $(y_1,\ldots,y_m)$. For
simplicity of presentation, we will assume that $m>d$, but note that for most
results this can be easily relaxed.

Given a vector of predicted values $\bv$ on the training set (or a matrix $V$
in a multi-class prediction setting), we use $\ell(\bv,\by)$ to denote the
training error, which is assumed to be a convex function of $\bv$. Some
examples include: \begin{itemize} \item
  Squared loss:
  $\ell(\bv,\by)=\frac{1}{m}\norm{\bv-\by}^2$ \item Hinge loss:
  $\ell(\bv,\by)=\frac{1}{m}\sum_{i=1}^{m}\max\{0,1-y_i v_i\}$ \item
  Logistic loss:
  $\ell(\bv,\by)=\frac{1}{m}\sum_{i=1}^{m}\log(1+\exp(-y_i v_i))$
    \item Multiclass hinge loss: $\ell(V,\by)=\frac{1}{m}\sum_{i=1}^{m}\max\{0,1+\max_{j\neq y_i}V_{i,j}-V_{i,y_i}\}$ (here, $V_{i,j}$ is the confidence score for instance $i$ being in class $j$)
\end{itemize}
Moreover, in the context of linear predictors, we can consider regularized
loss functions, where we augment the loss by a regularization term such as
$\frac{\lambda}{2} \norm{\bw}^2$ (where $\bw$ is the linear predictor) for
some parameter $\lambda>0$.

Multivariate polynomials are functions over $\reals^d$, of the form
\begin{equation}\label{eq:polydef}
\bp(\bx) = \sum_{i=0}^{\Delta}\sum_{\balpha^{(i)}} w_{\balpha^{(i)}}\prod_{l=1}^{d}x_l^{\alpha^{(i)}_l},
\end{equation}
where $\balpha^{(i)}$ ranges over all $d$-dimensional vectors of positive
integers, such that $\sum_{l=1}^{d}\alpha^{(i)}_l = i$, and $\Delta$ is the
degree of the polynomial. Each term $\prod_{l=1}^{d}x_l^{\alpha^{(i)}_l}$ is
a monomial of degree $i$.

To represent our network, we let $n^i_j(\cdot)$ refer to the $j$-th node in
the $i$-th layer, as a function of its inputs. In our algorithm, the function
each node computes is always either a linear function, or a weighted product
of two inputs:
\[
(z_1,z_2) ~\mapsto~ w z_1 z_2 ~,
\]
where $w \in \reals$. The depth of the network corresponds to the number of
layers, and the width corresponds to the largest number of nodes in any
single layer.

\section{The Basis Learner: Algorithm and
  Analysis}\label{sec:algorithmanalysis}

We now turn to develop our Basis Learner algorithm, as well as the
accompanying analysis. We do this in three stages: First, we derive a generic
and idealized version of our algorithm, which runs in polynomial time but is
not very practical; Second, we analyze its properties in terms of time
complexity, training error, etc.; Third, we discuss and analyze a more
realistic variant of our algorithm, which also enjoys some theoretical
guarantees, generalizes better, and is more flexible in practice.

\subsection{Generic Algorithm}

Recall that our goal is to learn polynomial predictors, using a deep
architecture, based on a training set with instances $\bx_1,\ldots,\bx_m$.
However, let us ignore for now the learning aspect and focus on a
\emph{representation} problem: how can we build a network capable of
representing the values of \emph{any} polynomial over the instances?

At first glance, this may seem like a tall order, since the space of all
polynomials is not specified by any bounded number of parameters. However,
our first crucial observation is that we care (for now) only about the values
on the $m$ training instances. We can represent these values as
$m$-dimensional vectors in $\reals^m$. Moreover, we can identify each
polynomial $\bp$ with its values on the training instances, via the linear
projection
\[
\bp \mapsto (\bp(\bx_1),\ldots,\bp(\bx_m)).
\]
Since the space of all polynomials can attain any set of values on a finite
set of distinct points \cite{GasSa00}, we get that polynomials span
$\reals^m$ via this linear projection. By a standard result from linear
algebra, this immediately implies that there are $m$ polynomials
$\bp_1,\ldots,\bp_m$, such that
$\{(\bp_i(\bx_1),\ldots,\bp_i(\bx_m))\}_{i=1}^{m}$ form a basis of $\reals^m$
- we can write any set of values $(y_1,\ldots,y_m)$ as a linear combination
of these. Formally, we get the following:
\begin{lemma}
Suppose $\bx_1,\ldots,\bx_m$ are distinct. Then there exist $m$ polynomials
$\bp_1,\ldots,\bp_m$, such that:
$\{(\bp_i(\bx_1),\ldots,\bp_i(\bx_m))\}_{i=1}^{m}$ form a basis of
$\reals^m$.

Hence, for any set of values $(y_1,\ldots,y_m)$, there is a coefficient
vector $(w_1,\ldots,w_m)$, so that $\sum_{i=1}^{m}w_i \bp_i(\bx_j)=y_j$ for
all $j=1,\ldots,m$.

\end{lemma}
This lemma implies that if we build a network, which computes such $m$
polynomials $\bp_1,\ldots,\bp_m$, then we can train a simple linear
classifier on top of these outputs, which can attain any target values over
the training data.

While it is nice to be able to express any target values $(y_1,\ldots,y_m)$
as a function of the input instances $(\bx_1,\ldots,\bx_m)$, such an
expressive machine will likely lead to overfitting. Our generic algorithm
builds a deep network such that the nodes of the first $r$ layers form a
basis of all values attained by degree-$r$ polynomials. Therefore, we start
with a simple network, which might have a large bias but will tend not to
overfit (i.e. low variance), and as we make the network deeper and deeper we
gradually decrease the bias while increasing the variance. Thus, in
principle, this algorithm can be used to train the natural curve of solutions
that can be used to control the bias-variance tradeoff.

It remains to describe how we build such a network. First, we show how to
construct a basis which spans all values attained by degree-1 polyonomials
(i.e. linear functions). We then show how to enlarge this to a basis of all
values attained by degree-2 polynomials, and so on. Each such enlargement of
the degree corresponds to another layer in our network. Later, we will prove
that each step can be calculated in polynomial time and the whole process
terminates after a polynomial number of iterations.

%While we do \emph{not} need to make any special assumptions on the data
%structure, we note that our idealized algorithm works particularly well when
%certain algebraic assumptions are met. As the most simple example, consider
%the case where the data lies on an affine manifold. In that case, a small
%number of degree-$1$ polynomials span all linear functions. A thin first
%layer (which spans degree-$1$ polynomials) will not result in a deterioration
%of the training error. This continues to deeper layers; If our data lies for
%example on the unit sphere, then choosing thinner higher level layers (that
%correspond to higher degree polynomials) will not affect the training error,
%as we simply remove exponentially many polynomials that vanish on the data
%set.

\subsubsection{Constructing the First Layer}

The set of values attained by degree-1 polynomials (linear) functions over
the data is
\begin{equation}\label{eq:beginset}
\left\{\left(\inner{\bw,[1~ \bx_1]},\ldots,\inner{\bw,[1~\bx_m]}\right):\bw\in \reals^{d+1}\right\},
\end{equation}
which is a $d+1$-dimensional linear subspace of $\reals^m$. Thus, to
construct a basis for it, we only need to find $d+1$ vectors
$\bw_1,\ldots,\bw_{d+1}$, so that the set of vectors
$\left\{\left(\inner{\bw_j,[1~\bx_1]},\ldots,\inner{\bw_j,[1~\bx_m]}\right)\right\}_{j=1}^{d+1}$
are linearly independent. This can be done in many ways. For example, one can
construct an orthogonal basis to \eqref{eq:beginset}, using Gram-Schmidt or
SVD (equivalently, finding a $(d+1)\times (d+1)$ matrix $W$, so that $[1~X]W$
has orthogonal columns)\footnote{This is
  essentially the same as the first step of the VCA algorithm in
  \cite{VCA}. Moreover, it is very similar to performing Principal
  Component Analysis (PCA) on the data, which is a often a standard
  first step in learning. It differs from PCA in that the SVD is done
  on the augmented matrix $[1~X]$, rather than on a centered version
  of $X$. This is significant here, since the columns of a centered
  data matrix $X$ cannot express the $\mathbf{1}$
  vector, hence we cannot express the constant $1$ polynomial on the data.\label{footnote1}}. At this stage, our focus is to present our
approach in full generality, so we avoid fixing a specific basis-construction
method.

Whatever basis-construction method we use, we end up with some linear
transformation (specified by a matrix $W$), which maps $[1~X]$ into the
constructed basis. The columns of $W$ specify the $d+1$ linear functions
forming the first layer of our network: For all $j=1,\ldots,d+1$, the $j$'th
node of the first layer is the function
\[
n^1_j(\bx) = \inner{W_j,[1~X]},
\]
and we have the property that
$\{(n^1_j(\bx_1),\ldots,n^1_j(\bx_m))\}_{j=1}^{d+1}$ is a basis for all
values attained by degree-1 polynomials over the training data. We let $F^1$
denote the $m\times (d+1)$ matrix\footnote{If the data lies in a subspace of
$\reals^d$ then the number of columns of $F^1$ will be the dimension of this
subspace plus $1$.} whose columns are the vectors of this set, namely,
$F^1_{i,j} = n^1_j(\bx_i)$.

\subsubsection{Constructing The Second Layer}

So far, we have a one-layer network whose outputs span all values attained by
linear functions on the training instances. In principle, we can use the same
trick to find a basis for degree-$2,3,\ldots$ polynomials: For any degree
$\Delta$ polynomial, consider the space of all values attained by such
polynomials over the training data, and find a spanning basis. However, we
quickly run into a computational problem, since the space of all degree
$\Delta$ polynomials in $\reals^d$ ($d>1$) increases exponentially in
$\Delta$, requiring us to consider exponentially many vectors. Instead, we
utilize our deep architecture to find a \emph{compact} representation of the
required basis, using the following simple but important observation:
\begin{lemma}\label{lem:decompose} Any degree $t$ polynomial can be
  written as \[ \sum_{i}\bg_i(\bx) \bh_i(\bx)+\bk(\bx), \] where
  $\bg_i(\bx)$ are degree-1 polynomials, $\bh_i(\bx)$ are
  degree-$(t-1)$ polynomials, and $\bk(\bx)$ is a polynomial of degree
  at most $t-1$. \end{lemma} \begin{proof} Any polynomial of degree
  $t$ can be written as a weighted sum of monomials of degree $t$,
  plus a polynomial of degree $\leq t-1$. Moreover, any monomial of
  degree $t$ can be written as a product of a monomial of degree
  $t-1$ and a monomial of degree $1$. Since $(t-1)$-degree monomials
  are in particular $(t-1)$-degree polynomials, the result follows.
\end{proof}

The lemma implies that any degree-2 polynomial can be written as the sum of
products of degree-1 polynomials, plus a degree-1 polynomial. Since the nodes
at the first layer of our network span all degree-1 polynomials, they in
particular span the polynomials $\bg_i,\bh_i,\bk$, so it follows that any
degree-2 polynomial can be written as
\begin{align*}
& \sum_{i} \left(\sum_j \alpha^{(\bg_i)}_j n^1_j(\bx)\right) \left(\sum_r \alpha^{(\bh_i)}_r n^1_r(\bx)\right) + \left(\sum_j \alpha^{(\bk)}_j n^1_j(\bx)\right) \\
&= \sum_{j,r} n^1_j(\bx) n^1_r(\bx) \left( \sum_i \alpha^{(\bg_i)}_j \alpha^{(\bh_i)}_r\right) + \sum_j n^1_j(\bx)  \left(\alpha^{(\bk)}_j \right)~,
\end{align*}
where all the $\alpha$'s are scalars. In other words, the vector of values
attainable by any degree-2 polynomial is in the span of the vector of values
attained by nodes in the first layer, and products of the outputs of every
two nodes in the first layer.

Let us now switch back to an algebraic representation. Recall that in
constructing the first layer, we formed a matrix $F^1$, whose columns span
all values attainable by degree-1 polynomials. Then the above implies that
the matrix $[F~\tilde{F}^2]$, where
\[
\tilde{F}^2 ~=~ \left[(F^1_1\circ F^1_1)~\cdots~(F^1_{1}\circ F^1_{|F_1|})~\cdots~(F^1_{|F^1|}\circ F^1_{1})~\cdots~(F^1_{|F^1|}\circ F^1_{|F_1|})\right],
\]
spans all possible values attainable by degree-2 polynomials. Thus, to get a
basis for the values attained by degree-2 polynomials, it is enough to find
some column subset $F^2$ of $\tilde{F}^2$, so that $[F~F^2]$'s columns are a
linearly independent basis for $[F~\tilde{F}^2]$'s columns. Again, this basis
construction can be done in several ways, using standard linear algebra (such
as a Gram-Schmidt procedure or more stable alternative methods). The columns
of $F^2$ (which are a subset of the columns of $\tilde{F}^2$) specify the 2nd
layer of our network: each such column, which corresponds to (say) $F^1_i
\circ F^1_j$, corresponds in turn to a node in the 2nd layer, which computes
the product of nodes $n^1_i(\cdot)$ and $n^1_j(\cdot)$ in the first layer.
We now redefine $F$ to be the augmented matrix $[F~F^2]$.

\subsubsection{Constructing Layer 3,4,\ldots}

It is only left to repeat this process. At each iteration $t$, we maintain a
matrix $F$, whose columns form a basis for the values attained by all
polynomials of degree $\leq t-1$. We then consider the new matrix
\[
\tilde{F}^t~=~ \left[(F^{t-1}_1\circ F^1_1)~\cdots~(F^{t-1}_{1}\circ F^1_{|F_1|})~\cdots~(F^{t-1}_{|F^{t-1}|}\circ F^1_{1})~\cdots~(F^{t-1}_{|F^{t-1}|}\circ F^1_{|F_1|})\right],
\]
and find a column subset $F^t$ so that the columns of $[F~F^t]$ form a basis
for the columns of $[F~\tilde{F}^t]$. We then redefine $F:=[F~F^t]$, and are
assured that the columns of $F$ span the values of all polynomials of degree
$\leq t$ over the data. By adding this newly constructed layer, we get a
network whose outputs form a basis for the values attained by all polynomials
of degree $\leq t$ over the training instances.

To maintain numerical stability, it may be desirable to multiply each column
of $F^t$ by a normalization factor, e.g. by scaling each column $F^t_i$ so
that the second moment $\frac{1}{m}\norm{F^t_i}^2$ across the column is $1$
(otherwise, the iterated products may make the values in the matrix very
large or small). Overall, we can specify the transformation from
$\tilde{F}_t$ to $F^t$ via a matrix $W$ of size $|F^{t-1}|\times |F^1|$, so
that for any $r=1,\ldots,|F^t|$, \[ F^t_r := W_{i(r),j(r)}F^{t-1}_{i(r)}
\circ F^{1}_{j(j)}. \]

As we will prove later on, if at any stage the subspace spanned by $[F~ F^t]$
is the same as the subspace spanned by $[F~ \tilde{F}^t]$, then our network
can span the values of all polynomials of any degree over the training data,
and we can stop the process.

The process (up to the creation of the output layer) is described in Figure
\ref{alg:main}, and the resulting network architecture is shown in Figure
\ref{fig:architecture}. We note that the resulting network has a feedforward
architecture. The connections, though, are not only between adjacent layers,
unlike many common deep learning architectures. Moreover, although we refrain
from fixing the basis creation methods at this stage, we provide one possible
implementation in Figure \ref{fig:createbasisexample}. We emphasize, though,
that other variants are possible, and the basis construction method can be
different at different layers.

\begin{figure}
\begin{center}
\fbox{
\begin{Balgorithm}
  Initialize $F$ as an empty matrix, and $\tilde{F}^1 := [\mathbf{1}~~ X]$\\
  $(F^1,W^1) := \texttt{\textbf{BuildBasis$\mathbf{^1}$}}(\tilde{F}^1)$\+\\
  // Columns of $F^1$ are linearly independent, and $F^1 = \tilde{F}^1 W^1$\-\\
  \textbf{Create first layer:} $\forall i\in \{1,\ldots,|F^1|\}$,~ $n^{1}_i(\bx) := \inner{W^1_i, [1~\bx]}$\\
  $F := F^1$\\
  \textbf{For} $t=2,3,\ldots$\+\+\\
  \textbf{Create candidate output layer:} $(n^{\text{output}}(\cdot),\texttt{error}):=\texttt{\textbf{OutputLayer}}(F)$\\
  \textbf{If} $\texttt{error}$ sufficiently small, \textbf{break}\\
  $\tilde{F}^t := \left[\left(F^{t-1}_1\circ F^1_1\right)~~~\left(F^{t-1}_1\circ F^1_2\right)~~~\ldots~~~\left(F^{t-1}_{|F^{t-1}|}\circ F^{1}_{|F^1|}\right)\right]$\\
  $(F^t,W^t) := \texttt{\textbf{BuildBasis$\mathbf{^t}$}}(F,\tilde{F}^t)$\+\\
  // Columns of $[F~F^t]$ are linearly independent \\
  // $W^t$ is such that $F^t_r = W^t_{i(r),j(r)}(F^{t-1}_{i(r)} \circ F^{1}_{j(r)})$\-\\
  \textbf{If} $|F^t|=0$, \textbf{break}\\
  \textbf{Create layer t:} For each non-zero element $W_{i(r),j(r)}$ in $W$, $r=1,..,|F^t|$,\+\\
  $n^t_r(\cdot) := W_{i(r),j(r)}n^{t-1}_{i(r)}(\cdot)n^1_{j(r)}(\cdot)$\-\\
  $F := [F ~~ F^t]$
\end{Balgorithm}
} \vskip 1cm

\texttt{\textbf{OutputLayer}}(F) \fbox { \begin{Balgorithm}
    $\bw := \arg\min_{\bw\in \reals^{|F|}}\ell(F\bw,\by)$\\
    $n^{\text{output}}(\cdot):=\inner{\bw,\bn(\cdot)}$,\+\\
    where $\bn(\cdot)=\left(n^{1}_1(\cdot),n^{1}_2(\cdot),\ldots,n^{t-1}_{|F^{t-1}|}(\cdot)\right)$ consists of the outputs of all nodes in the network\-\\
    Let $\texttt{error}$ be the error of $n^\text{output}(\cdot)$ on a validation data set\\
    \textbf{Return} $(n^\text{output}(\cdot),\texttt{error})$
\end{Balgorithm}
}

\end{center}
\caption{The Basis Learner algorithm. The top box is the main algorithm,
which constructs the network, and the bottom box is the output layer
construction procedure. At this stage, \texttt{BuildBasis$^1$} and
\texttt{BuildBasis$^t$} are not fixed, but we provide one possible
implementation in Figure \ref{fig:createbasisexample}. } \label{alg:main}
\end{figure}

\begin{figure}
\begin{center}
\includegraphics[scale=0.4]{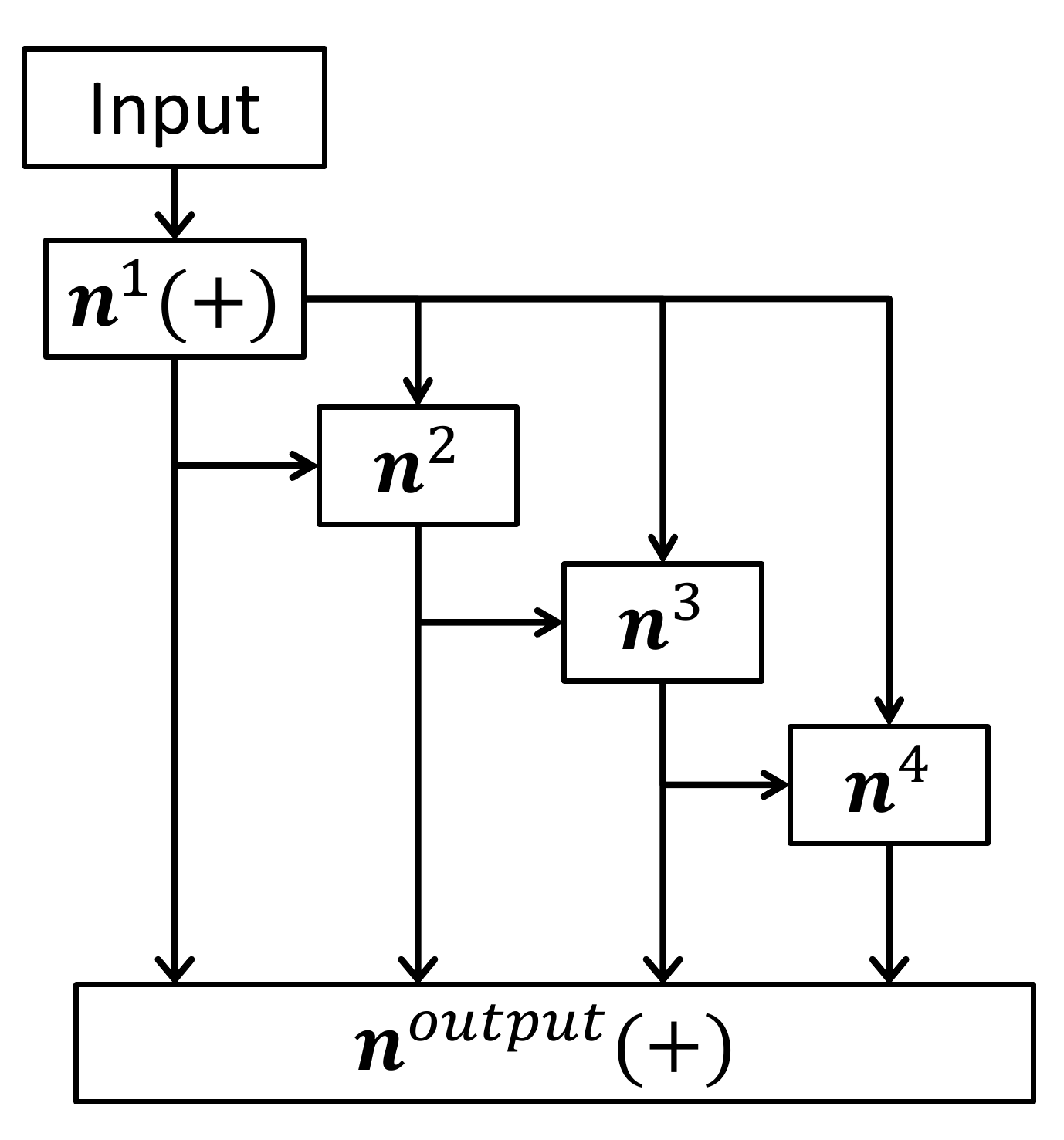}
\end{center}
\caption{Schematic diagram of the network's architecture, for polynomials of degree $4$. Each element represents a layer of nodes, as specified in Figure \ref{alg:main}. $(+)$ represent a layer of nodes which compute functions of the form $n(\bz) = \sum_i w_i z_i$, while other layers consist of nodes which compute functions of the form $n(\bz) = n((z_{i(1)},z_{i(2)})) = w z_{i(1)}z_{i(2)}$. In the diagram, computation moves top to bottom and left to right.}
\label{fig:architecture}
\end{figure}

\begin{figure}
\begin{center}
\textbf{\texttt{BuildBasis$\mathbf{^1}$}}($\tilde{F}^1$) - \textbf{example}\\
\fbox{
\begin{Balgorithm}
  Compute SVD: $\tilde{F}^1 = LDW^\top$\\
  Delete columns $W_i$ where $D_{i,i}=0$\\
  $B := \tilde{F}^1 W$\\
  \textbf{For} $i=1,\ldots,|W|$\+\+\\
  $b:= \sqrt{m}/\norm{B_i}$~;~
  $B_i := b B_i$~;~
  $W_i := b W_i$\-\-\\
  \textbf{Return} $(B,W)$ \end{Balgorithm} } \end{center}

\begin{center}
\textbf{\texttt{BuildBasis$\mathbf{^t}$}}($F,\tilde{F}^t$) - \textbf{example}\\
\fbox{
\begin{Balgorithm}
  Initialize $F^t:=[~]$, $W:=0$\\
  Compute orthonormal basis $O^F$ of $F$'s columns\+\\
  // Computed from previous call to \texttt{BuildBasis$^t$},\\
  // or directly via QR or SVD decomposition of $F$\-\\
  \textbf{For} $r=1,\ldots,|\tilde{F}^t|$\+\\
  $\bc:=\tilde{F}^t_r-O^F(O^F)^\top \tilde{F}^t_r$\\
  \textbf{If} $\norm{\bc} ~>~ tol$\+\\
  $F^t ~:=~ \left[F^t~~\frac{\sqrt{m}}{\norm{\tilde{F}^t_r}}\tilde{F}^t_r\right]$\\
  $W_{i(r),j(r)}= \sqrt{m}/\norm{\tilde{F}^t_r}$\+\\
  // $i(r),j(r)$ are those for which $\tilde{F}^t_r = F^{t-1}_{i(r)}\circ F^{1}_{j(r)}$\-\\
  $O^F:=\left[O^F~\frac{1}{\norm{\bc}}\bc\right]$\-\-\\
  \textbf{Return} $(F^t,W)$ \end{Balgorithm} } \end{center}
\caption{Example Implementations of the
  \texttt{BuildBasis$\mathbf{^1}$} and
  \texttt{BuildBasis$\mathbf{^t}$} procedures.
  \texttt{BuildBasis$\mathbf{^1}$} is implemented to return an
  orthogonal basis for $\tilde{F}^1$'s columns via SVD, while
  \texttt{BuildBasis$\mathbf{^t}$} uses a Gram-Schmidt procedure to
  find an appropriate columns subset of $\tilde{F}^t$, which together
  with $F$ forms a basis for $[F~\tilde{F}^t]$'s columns. In the pseudo-code,
  \emph{tol} is a tolerance parameter (e.g. machine precision).}
\label{fig:createbasisexample} \end{figure}

\paragraph{Constructing the Output Layer}
After $\Delta-1$ iterations (for some $\Delta$), we end up with a matrix $F$,
whose columns form a basis for all values attained by polynomials of degree
$\leq \Delta-1$ over the training data. Moreover, each column is exactly the
values attained by some node in our network over the training instances. On
top of this network, we can now train a simple linear predictor $\bw$,
miniming some convex loss function $\bw\mapsto \ell(F\bw,\by)$. This can be
done using any convex optimization procedure. We are assured that for any
polynomial of degree at most $\Delta-1$, there is some such linear predictor
$\bw$ which attains the same value as this polynomial over the data. This
linear predictor forms the output layer of our network.

As mentioned earlier, the inspiration to our approach is based on \cite{VCA},
which present an incremental method to efficiently build a basis for
polynomial functions. In particular, we use the same basic ideas in order to
ensure that after $t$ iterations, the resulting basis spans all polynomials
of degree at most $t$. While we owe a lot to their ideas, we should also
emphasize the differences: First, the emphasis there is to find a set of
generators for the ideal of polynomials vanishing on the training set.
Second, their goal there has nothing to do with deep learning, and the result
of the algorithm is a basis rather than a deep network. Third, they build the
basis in a different way than ours (forcing orthogonality of the basis
components), which does not seem as effective in our context (see end of
section \secref{sec:experiments}). Fourth, the practical variant of our
algorithm, which is described further on, is very different than the methods
used in \cite{VCA}.

Before continuing with the analysis, we make several important remarks:

\begin{remark}[Number of layers does not need to be fixed in advance]
Each iteration of the algorithm corresponds to another layer in the network.
However, note that we \emph{do not} need to specify the number of iterations.
Instead, we can simply create the layers one-by-one, each time attempting to
construct an output layer on top of the existing nodes. We then check the
performance of the resulting network on a validation set, and stop once we
reach satisfactory performance. See Figure \ref{fig:architecture} for
details.
\end{remark}

\begin{remark}[Flexibility of loss function]
Compared to our algorithm, many standard deep learning algorithms are more
constrained in terms of the loss function, especially those that directly
attempt to minimize training error. Since these algorithms solve hard,
non-convex problems, it is important that the loss will be as ``nice'' and
smooth as possible, and they often focus on the squared loss (for example,
the famous backpropagation algorithm \cite{rumelhart2002learning} is tailored
for this loss). In contrast, our algorithm can easily work with any convex
loss.
\end{remark}

\begin{remark}[Choice of Architecture]\label{remark:architecture}
In the intermediate layers we proposed constructing a basis for the columns
of $[F~ \tilde{F}^t]$ by using the columns of $F$ and a \emph{column subset}
of $\tilde{F}^t$. However, this is not the only way to construct a basis. For
example, one can try and find a full linear transformation $W^t$ so that the
columns of $[F\tilde{F}^t]W^t$ form an orthogonal basis to $[F~\tilde{F}^t]$.
However, our approach combines two important advantages. On one hand, it
creates a network with few connections where most nodes depend on the inputs
of only two other nodes. This makes the network very fast at test-time, as
well as better-generalizing in theory and in practice (see
\secref{sec:samplecomplexity} and \secref{sec:experiments} for more details).
On the other hand, it is still sufficiently expressive to \emph{compactly}
represent high-dimensional polynomials, in a product-of-sums form, whose
expansion as an explicit sum of monomials would be prohibitively large. In
particular, our network computes functions of the form $\bx\mapsto \sum_j
\alpha_j \prod_i(b^j_i+\inner{\bw^j_i,\bx})$, which involve exponentially
many monomials. The ability to compactly represent complex concepts is a
major principle in deep learning \cite{Beng09}. This is also why we chose to
use a linear transformation in the first layer - if all non-output layers
just compute the product of two outputs from the previous layers, then the
resulting predictor is limited to computing polynomials with a small number
of monomials.
\end{remark}

\begin{remark}[Connection to Algebraic Geometry]
  Our algorithm has some deep connections to algebraic geometry and
  interpolation theory. In particular, the problem of finding a basis
  for polynomial functions on a given set has been well studied in
  these areas for many years. However, most methods we are aware of -
  such as construction of Newton Basis polynomials or multivariate
  extensions of standard polynomial interpolation methods
  \cite{GasSa00} - are not computationally efficient, i.e. polynomial
  in the dimension $d$ and the polynomial degree $\Delta$. This is because
  they are based on explicit handling of monomials, of which there are
  $\binom{d+\Delta}{d}$. Efficient algorithms have been proposed for
  related problems, such as the Buchberger-M\"{o}ller algorithm for
  finding a set of generators for the ideal
  of polynomials vanishing on a given set (see
  \cite{MaMoMo93,AbBiKrRo00,VCA} and references therein). In a sense,
  our deep architecture is ``orthogonal" to this approach, since we focus
  on constructing a bsis for polynomials that \emph{do not} vanish on the set of
  points. This enables us to find an efficient, \emph{compact}
  representation, using a deep architecture, for getting arbitrary
  values over a training set.
\end{remark}

\subsection{Analysis}

After describing our generic algorithm and its derivation, we now turn to
prove its formal properties. In particular, we show that its runtime is
polynomial in the training set size $m$ and the dimension $d$, and that it
can drive the training error all the way to zero. In the next section, we
discuss how to make the algorithm more practical from a computational and
statistical point of view.

\begin{theorem}\label{thm:ideal}
Given a training set $(\bx_1,y_1),\ldots,(\bx_m,y_m)$, where
$\bx_1,\ldots,\bx_m$ are distinct points in $\reals^d$, suppose we run the
algorithm in Figure \ref{alg:main}, constructing a network of total depth
$\Delta$. Then:
\begin{enumerate}
\item\label{i:1} $|F|\leq m$, $|F^1|\leq d+1$, $\max_t |F^t|\leq m$,
    $\max_t |\tilde{F}^t|\leq m(d+1)$. \item\label{i:2} The algorithm
    terminates after at most $\min\{m-1,\Delta-2\}$ iterations of the For
    loop. \item\label{i:3} Assuming (for simplicity) $d\leq m$, the
    algorithm can be implemented using at most $\Ocal(m^2)$ memory and
    $\Ocal(d m^4)$ time, plus the polynomial time required to solve the
    convex optimization problem when computing the output layer.
\item\label{i:4} The network constructed by the algorithm has at most
    $\min\{m+1,\Delta\}$ layers, width at most $m$, and total number of
    nodes at most $m+1$. The total number of arithmetic operations (sums
    and products) performed to compute an output is $\Ocal(m+d^2)$.
\item\label{i:5} At the end of iteration $t$, $F$'s columns span all values
    attainable by polynomials of degree $\leq t$ on the training instances.
    \item\label{i:6} The training error of the network created by the
        algorithm is monotonically decreasing in $\Delta$. Moreover, if
        there exists some vector of prediction values $\bv$ such that
        $\ell(\bv,\by)=0$, then after at most $m$ iterations, the training
        error will be $0$.
\end{enumerate}
\end{theorem}
In item \ref{i:6}, we note that the assumption on $\ell$ is merely to
simplify the presentation. A more precise statement would be that we can get
the training error arbitrarily close to $\inf_{\bv}\ell(\bv,\bp)$ - see the
proof for details.

\begin{proof}
  The theorem is mostly an easy corollary of the derivation.

  As to item \ref{i:1}, since we maintain the $m$-dimensional columns
  of $F$ and each $F^t$ to be linearly independent, there cannot be
  more than $m$ of them. The bound on $|F^1|$ follows by construction
  (as we orthogonalize a matrix with $d+1$ columns), and the bound on
  $|\tilde{F}^t|$ now follows by definition of $\tilde{F}^t$.

  As to item \ref{i:2}, the algorithm always augments $F$ by $F^t$,
  and breaks whenever $|F^t|=0$. Since $F$ can have at most $m$
  columns, it follows the algorithm cannot run more than $m$
  iterations. The algorithm also terminates after at most $\Delta-2$
  iterations, by definition.

  As to item \ref{i:3}, the memory bound follows from the bounds on
  the sizes of $F,F^t$, and the associated sizes of the constructed
  network. Note that $\tilde{F}^t$ can require as much as
  $\Ocal(dm^2)$ memory, but we don't need to store it explicitly - any
  entry in $\tilde{F}^t$ is specified as a product of two entries in
  $F^1$ and $F^{t-1}$, which can be found and computed on-the-fly in
  $\Ocal(1)$ time. As to the time bound, each iteration of our
  algorithm involves computations polynomial in $m,d$, with the
  dominant factors being the \texttt{BuildBasis$^t$} and
  \texttt{BuildBasis$^1$}. The time bounds follow from the the
  implementations proposed in Figure \ref{fig:createbasisexample},
  using the upper bounds on the sizes of the relevant matrices, and
  the assumption that $d\leq m$.

  As to item \ref{i:4}, it follows from the fact that in each
  iteration, we create layer $t$ with at most $|F^t|$ new nodes, and
  there are at most $\min\{m-1,\Delta-2\}$ iterations/layers excluding
  the input and output layers. Moreover, each node in our network
  (except the output node) corresponds to a column in $|F|$, so there
  are at most $m$ nodes plus the output nodes. Finally, the network
  computes a linear transformation in $\reals^d$, then at most $m$
  nodes perform $2$ products each, and a final output node computes a
  weighted linear combination of the output of all other nodes (at
  most $m$) - so the number of operations is $\Ocal(m+d^2)$.

  As to item \ref{i:5}, it follows immediately by the derivation
  presented earlier.

  Finally, we need to show item \ref{i:6}. Recall from the derivation
  that in the output layer, we use the linear weights $\bw$ which
  minimize $\ell(F\bw,\by)$. If we increase the depth of our
  constructed network, what happens is that we augment $F$ by more and
  more linearly independent columns, the initial columns being exactly
  the same. Thus, the size of the set of prediction vectors
  $\{F\bw:\bw\in \reals^{|F|}\}$ only increases, and the training
  error can only go down.

  If we run the algorithm till $|F|=m$, then the columns of $F$ span
  $\reals^m$, since the columns of $F$ are linearly independent. Hence
  $\{F\bw:\bw\in \reals^{|F|}\}=\reals^m$. This implies that we can
  always find $\bw$ such that $F\bw=\bv$, where $\ell(\bv,\by)=0$, so
  the training error is zero. The only case left to treat is if the
  algorithm stops when $|F|<m$. However, we claim this can't happen.
  This can only happen if $|F^t|=0$ after the basis construction
  process, namely that $F$'s columns already span the columns of
  $\tilde{F}^t$. However, this would imply that we can span the values
  of all degree-$t$ polynomials on the training instances, using
  polynomials of degree $\leq t-1$. But using \lemref{lem:decompose},
  it would imply that we could write the values of every degree-$t+1$
  polynomial using a linear combination of polynomials of degree $\leq
  t-1$. Repeating this, we get that the values of polynomials of
  degree $t+2,t+3,\ldots$ are all spanned by polynomials of degree
  $t-1$. However, the values of all polynomials of any degree over $m$
  distinct points must span $\reals^m$, so we must have $|F|=m$.
\end{proof}

An immediate corollary of this result is the following:

\begin{remark}[The Basis Learner is a Universal Learner]
Our algorithm is a universal algorithm, in the sense that as we run it for
more and more iterations, the training error provably decreases, eventually
hitting zero. Thus, we can get a curve of solutions, trading-off between the
training error on one hand and the size of the resulting network (as well as
the potential of overfitting) on the other hand.
%Furthermore, since any
%target predictor can be approximated by a polynomial of some degree $t$, it
%follows that with a sufficiently large training set, after performing $t$
%iterations of our algorithm we are guaranteed to learn a competitive
%predictor.
\end{remark}

% \begin{remark}[Actual Depth of Resulting Network]
%  If the network depth is not constrained, then the algorithm may construct as many as $m$ layers in the worst-case. However, it will generally be much smaller. To see why, recall that by our analysis, the first $t$ layers span all values attainable by degree-$t$ polynomials on the training instances. Moreover, the algorithm will stop once any set of values is attainable. Hence, the number of layers will equal the minimal degree required for multivariate polynomials of this degree to attain any value on the training instances. The required degree can be as high as $m$ (e.g. if all points are on the same line, and the problem reduces to the one-dimensional case), but will generally be much lower. Intuitively, this is because a degree-$n$ polynomials in $\reals^d$ are very expressive: each polynomial has ${\binom{d+n}{n}}$ parameters (since it is a weighted sum of $\binom{d+n}{n}$ monomials), hence it will often be able to quickly fit any values on a finite dataset, as $n$ increases. Unfortunately, as far as we know, there is no general characterization of the minimal required degree as a function of the point set (see \cite{GasSa00}, Section 1.2).
%\end{remark}

\subsection{Making the Algorithm Practical}\label{subsec:practical}

\begin{figure}
\begin{center}
\textbf{\texttt{BuildBasis$\mathbf{^1}$}}($\tilde{F}^1$) - \textbf{Width-Limited Variant}\\
\fbox{
\begin{Balgorithm}
  \textbf{Parameter:} Layer width $\gamma\geq 0$\\
  Compute SVD: $\tilde{F}^1 = LDW^\top$\\
  $W := [W_1~W_2~\cdots~W_{\gamma}]$\+\\
  // Assumed to be the columns corresponding to $\gamma$ largest non-zero singular values\-\\
  $B := \tilde{F}^1 W$\\
  \textbf{For} $i=1,\ldots,|W|$\+\+\\
  $b:= \sqrt{m}/\norm{B_i}$~;~
  $B_i := b B_i$~;~
  $W_i := b W_i$\-\-\\
  \textbf{Return} $(B,W)$ \end{Balgorithm} } \end{center}

\begin{center}
\textbf{\texttt{BuildBasis$\mathbf{^t}$}}($F,\tilde{F}^t$) - \textbf{Width-Limited Variant}\\
\fbox{
\begin{Balgorithm}
  \textbf{Parameter:} Layer width $\gamma\geq 0$, batch size $b$\\
  Let $V$ denote target value vector/matrix (see caption)\\
  Initialize $F^t:=[~]$, $W:=0$\\
  Compute orthonormal basis $O^F$ of $F$'s columns\+\\
  // Computed in the previous call to \texttt{BuildBasis$^t$},\\
  // or directly via QR or SVD decomposition of $F$\-\\
  $V:=V-O^F(O^F)^\top V$\\
  \textbf{For} $r=1,2,\ldots,(\gamma/b) $\+\\
  $C:=\tilde{F}^t-O^F(O^F)^\top\tilde{F}^t$\\
  $C_i := \frac{1}{\norm{C_i}}C_i$ for all $i=1,\ldots,|C|$\\
  Compute orthonormal basis $O^V$ of $V$'s columns\\
  Let $i(1),\ldots,i(b)$ be indices of the $b$ linearly independent columns \+\+\\
  of $(O^V)^\top C$ with largest positive norm\-\-\\
  \textbf{For} $r=i(1),i(2),\ldots,i(b)$:\+\\
  $F^t ~:=~ \left[F^t~~\frac{\sqrt{m}}{\norm{\tilde{F}^t_{r}}}\tilde{F}^t_r\right]$\\
  $W_{i(r),j(r)}= \sqrt{m}/\norm{\tilde{F}^t_r}$\+\\
  // $i(r),j(r)$ are those for which $\tilde{F}^t_r = F^{t-1}_{i(r)}\circ F^{1}_{j(r)}$\-\-\\
  Compute orthonormal basis $O^C$ of columns of $[C_{i(1)}~C_{i(2)}\cdots C_{i(b)}]$\\
  $O^F:=\left[O^F ~ O^C\right]$\\
  $V:=V-O^C(O^C)^\top V$\-\-\\
  \textbf{Return} $(F^t,W)$ \end{Balgorithm} } \end{center} \caption{
Practical width-limited implementations of the
\texttt{BuildBasis$\mathbf{^1}$} and \texttt{BuildBasis$\mathbf{^t}$}
procedures. \texttt{BuildBasis$\mathbf{^1}$} is implemented to return an
orthogonal partial basis for $\tilde{F}^1$'s, which spans the largest
singular vectors of the data. \texttt{BuildBasis$\mathbf{^t}$} uses a
supervised OLS procedure in order to pick a partial basis for
$[F~\tilde{F}^t]$, which is most useful for prediction. In the code, $V$
represents the vector of training set labels $(y_1,\ldots,y_m)$ for binary
classification and regression, and the indicator matrix
$V_{i,j}=\text{Ind}(y_i=j)$ for multiclass prediction. For simplicity, we
assume the batch size $b$ divides the layer width $\gamma$.}
\label{fig:createbasispractical} \end{figure}

While the algorithm we presented runs in provable polynomial time, it has
some important limitations. In particular, while we can always control the
depth of the network by early stopping, we do not control its width (i.e. the
number of nodes created in each layer). In the worst case, it can be as large
as the number of training instances $m$. This has two drawbacks:
\begin{itemize} \item The algorithm can
  only be used for small datasets - when $m$ is
  large, we might get huge networks, and running the algorithm will be
  computationally prohibitive, involving manipulations of matrices of
  order $m\times md$. \item Even ignoring computational constraints,
  the huge network which
  might be created is likely to overfit. \end{itemize}

To tackle this, we propose a simple modification of our scheme, where the
network width is explicitly constrained at each iteration. Recall that the
width of a layer constructed at iteration $t$ is equal to the number of
columns in $F^t$. Till now, $F^t$ was such that the columns of $[F~F^t]$ span
the column space of $[F~\tilde{F}^t]$. So if $|\tilde{F}^t|$ is large,
$|F^t|$ might be large as well, resulting in a wide layer with many new
nodes. However, we can give up on exactly spanning $\tilde{F}^t$, and instead
seek to ``approximately span'' it, using a smaller partial basis of bounded
size $\gamma$, resulting in a layer of width $\gamma$.

The next natural question is how to choose this partial basis. There are
several possible criterions, both supervised and unsupervised. We will focus
on the following choice, which we found to be quite effective in practice:
\begin{itemize} \item The first layer computes
  a linear transformation
  which transforms the augmented data matrix $[\mathbf{1}~X]$ into its
  first $\gamma$ leading singular vectors (this is closely akin - although not
  identical - to Principal Component Analysis (PCA) - see Footnote
  \ref{footnote1}).
    \item The next layers use a standard Orthogonal Least Squares procedure \cite{chenBillLu89} to greedily pick the columns of $\tilde{F}^t$ which seem most relevant for prediction. The intuition is that we wish to quickly decrease the training error, using a small number of new nodes and in a computationally cheap way. Specifically, for binary classification and regression, we consider the vector $\by=(y_1,\ldots,y_m)$ of training labels/target values, and  iteratively pick the column of $\tilde{F}^t$ whose residual (after projecting on the existing basis $F$) is most correlated with the residual of $\by$ (again, after projecting on the existing basis $F$). The column is then added to the existing basis, and the process repeats itself. A simple extension of this idea can be applied to the multiclass case. Finally, to speed-up the computation, we can process the columns of $\tilde{F}^t$ in mini-batches, where each time we find and add the $b$ ($b>1$) most correlated vectors before iterating.
\end{itemize}
These procedures are implemented via the subroutines
\texttt{BuildBasis${^1}$} and \texttt{BuildBasis${^t}$}, whereas the main
algorithm (Figure \ref{alg:main}) remains unaffected. A precise pseudo-code
appears in Figure \ref{fig:createbasispractical}. We note that in a practical
implementation of the pseudo-code, we do not need to explicitly compute the
potentially large matrices $C,\tilde{F}_t$ - we can simply compute each
column and its associated correlation score one-by-one, and use the list of
scores to pick and re-generate the most correlated columns.

We now turn to discuss the theoretical properties of this width-constrained
variant of our algorithm. Recall that in its idealized version, the Basis
Learner is guaranteed to eventually decrease the training error to zero in
all cases. However, with a width constraint, there are adversarial cases
where the algorithm will get ``stuck'' and will terminate before the training
error gets to zero. This may happen when $|F|<m$, and all the columns of
$\tilde{F}^t$ are spanned by $F$, so no new linearly independent vectors can
be added to $F$, $|F^t|$ will be zero, and the algorithm will terminate (see
Figure \ref{alg:main}). However, we have never witnessed this happen in any
of our experiments, and we can prove that this is indeed the case as long as
the input instances are in ``general position'' (which we shortly formalize).
Thus, we get a completely analogous result to \thmref{thm:ideal}, for the
more practical variant of the Basis Learner.

Intuitively, the general position condition we require implies that if we
take any two columns $F_i,F_j$ in $F$, and $|F|<m$ then the product vector
$F_i \circ F_j$ is linearly independent from the columns of $F$. This is
intuitively plausible, since the entry-wise product $\circ$ is a highly
non-linear operation, so in general there is no reason that $F_i \circ F_j$
will happen to lie exactly at the subspace spanned by $F$'s columns. More
formally, we use the following:
\begin{definition}\label{def:genposition}
  Let $\bx_1,\dots,\bx_m$ be a set of distinct points in
  $\mathbb{R}^d$. We say that $\bx_1,\dots,\bx_m$ are in
  \emph{M-general position} if for every $m$ monomials, $g_1,\ldots,g_m$, the $m\times m$ matrix $M$ defined as $M_{i,j}=g_j(\bx_i)$ has rank $m$.
\end{definition}

%Recall that a set of $m$ vectors in $\reals^d$ is said to be in linear general position if any subset of $k \le d$ vectors is linearly independent.
%\begin{definition}\label{def:genposition}
%  Let $\bx_1,\dots,\bx_m$ be a set of distinct points in
%  $\mathbb{R}^d$. We say that $\bx_1,\dots,\bx_m$ are in
%  \emph{M-general position} if for every $m$ monomials, $g_1,\ldots,g_m$, the set of points obtained by the projection $\bx \mapsto (g_1(\bx),\ldots,g_m(\bx))$ yields $m$ vectors in linear general position in $\reals^m$.
%\end{definition}

The following theorem is analogous to \thmref{thm:ideal}. The only difference
is in item \ref{j:5}, in which we use the M-general position assumption.
\begin{theorem}\label{thm:real}
Given a training set $(\bx_1,y_1),\ldots,(\bx_m,y_m)$, where
$\bx_1,\ldots,\bx_m$ are distinct points in $\reals^d$, suppose we run the
algorithm in Figure \ref{alg:main}, with the subroutines implemented in
Figure \ref{fig:createbasispractical}, using a uniform value for the width
$\gamma$ and batch size $b$, constructing a network of depth $\Delta$. Then:
\begin{enumerate}
\item\label{j:1} $|F|\leq \gamma\Delta$, $\max_t |F^t|\leq \gamma$, $\max_t
    |\tilde{F}^t|\leq \gamma^2$. \item\label{j:2} Assume (for simplicity)
    that $d\leq m$, and the case of regression or classification with a
    constant number of classes. Then the algorithm can be implemented using
    at most $\Ocal(m(d+\gamma\Delta))$ memory and $\Ocal(\Delta m(\gamma b
    + \Delta\gamma^4/b))$ time, plus the polynomial time required to solve
    the convex optimization problem when computing the output layer, and
    the SVD in \texttt{CreateBasis$^1$} (see remark below).
\item\label{j:3} The network constructed by the algorithm has at most
    $\Delta$ layers, with at most $\gamma$ nodes in each layer. The total
    number of nodes is at most $\min\{m,(\Delta-1)\gamma\}+1$. The total
    number of arithmetic operations (sums and products) performed to
    compute an output is $\Ocal(\gamma(d+\Delta))$.
    \item\label{j:4} The training error of the network created by the
        algorithm is monotonically decreasing in $\Delta$.
\item \label{j:5} If the rows of the matrix $B$ returned by the
    width-limited variant of \texttt{BuildBasis${^1}$} are in M-general
    position, $\Delta$ is unconstrained, and there exists some vector of
    prediction values $\bv$ such that $\ell(\bv,\by)=0$, then after at most
    $m$ iterations, the training error will be $0$.
\end{enumerate}
\end{theorem}
\begin{proof}
The proof of the theorem, except part \ref{j:5}, is a simple adaptation of
the proof of \thmref{thm:ideal}, using our construction and the remarks we
made earlier. So, it is only left to prove part \ref{j:5}. The algorithm will
terminate before driving the error to zero if at some iteration we have that
the columns of $\tilde{F}^t$ are spanned by $F$ and $|F|<m$. But, by
construction, this implies that there are $|F|\le m$ monomials such that if
we apply them on the rows of $B$, we obtain linearly dependent vectors. This
contradicts the assumption that the rows of $B$ are in M-general position and
concludes our proof.
\end{proof}

We note that in item \ref{j:2}, the SVD mentioned is over an $m\times (d+1)$
matrix, which requires $\Ocal(md^2)$ time to perform exactly. However, one
can use randomized approximate SVD procedures (e.g. \cite{halko2011finding})
to perform the computation in $\Ocal(md\gamma)$ time. While not exact, these
approximate methods are known to perform very well in practice, and in our
experiments we observed no significant degradation by using them in lieu of
exact SVD. Overall, for fixed $\Delta,\gamma$, this allows our Basis Learner
algorithm to construct the network in time \emph{linear} in the data size.

Overall, compared to \thmref{thm:ideal}, we see that our more practical
variant significantly lowers the memory and time requirements (assuming
$\Delta,\gamma$ are small compared to $m$), and we still have the property
that the training error decreases monotonically with the network depth, and
reduces to zero under mild conditions that are likely to hold on natural
datasets.

Before continuing, we again emphasize that our approach is quite generic, and
that the criterions we presented in this section, to pick a partial basis at
each iteration, are by no means the only ones possible. For example, one can
use other greedy selection procedures to pick the best columns in
\texttt{BuildBasis}${^t}$, as well as unsupervised methods. Similarly, one
can use supervised methods to construct the first layer. Also, the width of
different layers may differ. However, our goal here is \emph{not} to propose
the most sophisticated and best-performing method, but rather demonstrate
that using our approach, even with very simple regularization and greedy
construction methods, can have good theoretical guarantees and work well
experimentally. Of course, much work remains in trying out other methods.

\section{Sample Complexity}\label{sec:samplecomplexity}

So far, we have focused on how the network we build reduces the training
error. However, in a learning context, what we are actually interested in is
getting good generalization error, namely good prediction in expectation over
the distribution from which our training data was sampled.

We can view our algorithm as a procedure which given training data, picks a
network of width $\gamma$ and depth $\Delta$. When we use this network for
binary classification (e.g. by taking the sign of the output to be the
predicted label), a relevant measure of generalization performance is the
VC-dimension of the class of such networks. Luckily, the VC-dimension of
neural networks is a well-studied topic. In particular, by Theorem 8.4 in
\cite{AnBa02}, we know that any binary function class in Euclidean space,
which is parameterized by at most $n$ parameters and each function can be
specified using at most $t$ addition, multiplication, and comparison
operations, has VC dimension at most $\Ocal(nt)$. Our network can be
specified in this manner, using at most $\Ocal(\gamma(d+\Delta))$ operations
and parameters (see \thmref{thm:real}). This immediately implies a VC
dimension bound, which ensures generalization if the training data size is
sufficiently large compared to the network size. We note that this bound is
very generic and rather coarse - we suspect that it can be substantially
improved in our case. However, qualitatively speaking, it tells us that
reducing the number of parameters in our network reduces overfitting. This
principle is used in our network architecture, where each node in the
intermediate layers is connected to just $2$ other nodes, rather than (say)
all nodes in the previous layer.

As an interesting comparison, note that our network essentially computes a
$\Delta$-degree polynomial, yet the VC dimension of all $\Delta$-degree
polynomial in $\reals^d$ is ${d+\Delta \choose
  \Delta}$, which grows very fast with $d$ and $\Delta$
\cite{BenDavidLin98}. This shows that our algorithm can indeed generalize
better than directly learning high-degree polynomials, which is essentially
intractable both statistically and computationally.

It is also possible to prove bounds on scale-sensitive measures of
generalization (which are relevant if we care about the prediction values
rather than just their sign, e.g. for regression). For example, it is
well-known that the expected squared loss can be related to the empirical
squared loss over the training data, given a bound on the fat-shattering
dimension of the class of functions we are learning \cite{AnBa02}. Combining
Theorems 11.13 and 14.1 from \cite{AnBa02}, it is known that for a class of
networks such as those we are learning, the fat-shattering dimension is
upper-bounded by the VC dimension of a slightly larger class of networks,
which have an additional real input and an additional output node computing a
linear threshold function in $\reals^2$. Such a class of networks has a
similar VC dimension to our original class, hence we can effectively bound
the fat-shattering dimension as well.

\section{Relation to Kernel Learning}\label{sec:kernels}

Kernel learning (see e.g. \cite{scholkopf2002learning}) has enjoyed immense
popularity over the past 15 years, as an efficient and principled way to
learn complex, non-linear predictors. A kernel predictor is of the form
$\sum_{i}\alpha_i k(\bx_i,\cdot)$, where $\bx_1,\ldots,\bx_m$ are the
training instances, and $k(\cdot,\cdot)$ is a kernel function, which
efficiently computes an inner product $\inner{\Psi(\cdot),\Psi(\cdot)}$ in a
high or infinite-dimensional Hilbert space, to which data is mapped
implicitly via the feature mapping $\Psi$. In this section, we discuss some
of the interesting relationships between our work and kernel learning.

In kernel learning, a common kernel choice is the polynomial kernel,
$k(\bx,\bx')=(1+\inner{\bx,\bx'})^\Delta$. It is easy to see that predictors
defined via the polynomial kernel correspond to polynomial functions of
degree $\Delta$. Moreover, if the Gram matrix (defined as
$G_{i,j}=k(\bx_i,\bx_j)$) is full-rank, any values on the training data can
be realized by a kernel predictor: For a desired vector of values $\by$,
simply find the coefficient vector $\balpha$ such that $G\balpha = \by$, and
note that this implies that for any $j$, $\sum_i \alpha_i k(\bx_i,\bx_j) =
y_j$. Thus, when our algorithm is ran to completion, our polynomial network
can represent the same predictor class as kernel predictors with a polynomial
kernel. However, there are some important differences, which can make our
system potentially better:

\begin{itemize}
\item With polynomial kernels, one always has to manipulate an $m\times m$
    matrix, which requires memory and runtime scaling at least
    quadratically in $m$. This can be very expensive if $m$ is large, and
    hinders the application of kernel learning to large-scale data. This
    quadratic dependence on $m$ is also true at test time, where we need to
    explicitly use our $m$ training examples for prediction. In contrast,
    the size of our network can be controlled, and the memory and runtime
    requirements of our algorithm is only linear in $m$ (see
    \thmref{thm:real}). If we get good results with a moderately-sized
    network, we can train and predict much faster than with kernels. In
    other words, we get the potential expressiveness of polynomial kernel
    predictors, but with the ability to \emph{control} the training and
    prediction complexity, potentially requiring much less time and memory.
    \item With kernels, one has to specify the degree $\Delta$ of the
    polynomial kernel in advance before training. In contrast, in our
    network, the degree of the resulting polynomial predictor does not have
    to be specified in advance - each iteration of our algorithm increases
    the effective degree, and we stop when satisfactory performance is
    obtained. \item Learning with polynomial kernels corresponds to
    learning a linear combination over the set of polynomials
    $\{(1+\inner{\bx_i,\cdot})^\Delta\}_{i=1}^{m}$. In contrast, our
    network learns (in the output layer) a linear combination of a
    different set of polynomials, which is constructed in a different,
    data-dependent way. Thus, our algorithm uses a different and
    incomparable hypothesis class compared to polynomial kernel learning.
\item Learning with polynomial kernels can be viewed as a network of a
    shallow architecture as follows: Each node in the first layer
    corresponds to one support vector and applies the function $\bx \mapsto
    (1+\inner{\bx_i,\bx})^\Delta$. Then, the second layer is a linear
    combination of the outputs of the first layer. In contrast, we learn a
    deeper architecture. Some empirical evidence shows that deeper
    architectures may express complicated functions more compactly than
    shallow architectures \cite{Beng09,DeBe11}.
\end{itemize}

%A related interesting point, which we also mentioned earlier in the paper, is
%that our network can \emph{compactly} represent high-dimensional polynomials.
%Since the first layer of our network computes a linear transformation, and
%the next layers compute products of such terms, we end up with polynomial
%functions which are specified by a large number of product of linear terms.
%This product-of-sums representation allows us to compactly compute
%polynomials, whose explicit expansion as a sum of monomials is prohibitively
%large (see \eqref{eq:polydef}).  In kernel learning as well, the kernel
%function compactly represents an inner product in a very high or
%infinite-dimensional Hilbert space, for which an explicit representation is
%intractable.

%Finally, we note that there exist previous works which attempt to connect
%kernels and deep learning, but in a very different way than ours. In
%particular, \cite{chosa09} propose a kernel learning algorithm, where the
%kernel mimics computations in large random neural networks. However, the
%resulting predictor is still a kernel predictor. That paper also proposes a
%heuristic deep architecture, where kernel-PCA and feature selection methods
%are stacked on top of each other. In contrast, we propose a principled method
%to directly learn a deep architecture, which does not rely at all on kernel
%computations.

%\todo{We should also discuss the relation to kernel PCA}

\section{Experiments}\label{sec:experiments}

In this section, we present some preliminary experimental results to
demonstrate the feasibility of our approach. The focus here is not to show
superiority to existing learning approaches, but rather to illustrate how our
approach can match their performance on some benchmarks, using just a couple
of parameters and with no manual tuning.

%The important point is that our algorithm can successfully train a
%fundamentally deep learning system, using just a couple of parameters,
%relatively quickly, with no manual tuning, and with the authors having no
%prior experience in deep learning and no domain expertise. In contrast, the
%reported results for the benchmarks we consider, using existing deep learning
%techniques, appear to have involved many man-hours of painstaking manual
%tuning and heuristic parameter searches, and were performed by experienced
%researchers specializing in deep learning.

To study our approach, we used the benchmarks and protocol described in
\cite{larochelle2007empirical} \footnote{These datasets and
  experimental details are publicly available at
  \url{http://www.iro.umontreal.ca/\~lisa/twiki/bin/view.cgi/Public/DeepVsShallowComparisonICML2007\#Downloadable_datasets}}
. These benchmark datasets were designed to test deep learning systems, and
require highly non-linear predictors. They consist of $8$ datasets, where
each instance is a $784$-dimensional vector, representing normalized
intensity values of a $28\times 28$ pixel image. These datasets are as
follows: \begin{enumerate}
\item\label{d1} \texttt{MNIST-basic}: The well-known MNIST digit
  recognition
  dataset\footnote{\url{http://yann.lecun.com/exdb/mnist}}, where the
  goal is to identify handwritten digits in the image. \item\label{d2}
  \texttt{MNIST-rotated}: Same as \texttt{MNIST-basic},
  but with the digits randomly rotated. \item\label{d3}
  \texttt{MNIST-back-image}: Same as
  \texttt{MNIST-basic}, but with patches taken from unrelated
  real-world images in the background. \item\label{d4}
  \texttt{MNIST-back-random}: Same as
  \texttt{MNIST-basic}, but with random pixel noise in the background.
\item\label{d5} \texttt{MNIST-rotated+back-image}: Same as
  \texttt{MNIST-back-image}, but with the digits randomly rotated.
\item\label{d6} \texttt{Rectangles}: Given an image of a rectangle,
  determine whether its height is larger than its width.
\item\label{d7} \texttt{Rectangles-images}: Same as
  \texttt{Rectangles}, but with patches taken from unrelated
  real-world images in the background.
    \item\label{d8} \texttt{Convex}: Given images of various shapes, determine whether they are convex or not.
\end{enumerate}
All datasets consist of 12,000 training instances and 50,000 test instances,
except for the \texttt{Rectangles} dataset (1200/50000 train/test instances)
and the \texttt{Convex} dataset (8000/50000 train/test instances). We refer
the reader to \cite{larochelle2007empirical} for more precise details on the
construction used.

In \cite{larochelle2007empirical}, for each dataset and algorithm, the last
2000 examples of the training set was split off and used as a validation sets
for parameter tuning (except \texttt{Rectangles}, where it was the last 200
examples). The algorithm was then trained on the entire training set using
those parameters, and classification error on the test set was reported.

The algorithms used in \cite{larochelle2007empirical} involved several deep
learning systems: Two deep belief net algorithms (DBN-1 and DBN-3), a stacked
autoencode algorithm (SAA-3), and a standard single-hidden-layer,
feed-forward neural network (NNet). Also, experiments were ran on Support
Vector Machines, using an RBF kernel (SVM-RBF) and a polynomial kernel
(SVM-Poly).

We experimented with the practical variant of our Basis Learner algorithm (as
described in \subsecref{subsec:practical}), using a simple ,
publicly-available implementation in
MATLAB\footnote{\url{http://www.wisdom.weizmann.ac.il/~shamiro/code/BasisLearner.zip}}.
As mentioned earlier in the text, we avoided storing $\tilde{F}^t$, instead
computing parts of it as the need arose. We followed the same experimental
protocol as above, using the same split of the training set and using the
validation set for parameter tuning. For the output layer, we used stochastic
gradient descent to train a linear classifier, using a standard
$L_2$-regularized hinge loss (or the multiclass hinge loss for multiclass
classification). In the intermediate layer construction procedure
(\texttt{BuildBasis$^t$}), we fixed the batch size to $50$. We tuned the
following $3$ parameters:
\begin{itemize}
\item Network width $\gamma\in \{50,100,150,200,250,300\}$ \item Network
    depth $\Delta\in \{2,3,4,5,6,7\}$
    \item Regularization parameter $\lambda\in
        \{10^{-7},10^{-6.5},\ldots,10^1\}$
\end{itemize}
Importantly, we did \emph{not} need to train a new network for every
combination of these values. Instead, for every value of $\gamma$, we simply
built the network one layer at a time, each time training an output layer
over the layers so far (using the different values of $\lambda$), and
checking the results on a validation set. We deviated from this protocol only
in the case of the \texttt{MNIST-basic} dataset, where we allowed ourselves
to check $4$ additional architectures: The width of the first layer
constrained to be $50$, and the other layers are of width $100$,$200$,$400$
or $600$. The reason for this is that MNIST is known to work well with a PCA
preprocessing (where the data is projected to a few dozen principal
components). Since our first layer also performs a similar type of
processing, it seems that a narrow first layer would work well for this
dataset, which is indeed what we've observed in practice. Without trying
these few additional architectures, the test classification error for
$\texttt{MNIST-basic}$ is $4.32\%$, which is about $0.8\%$ worse than what is
reported below.

%Although it is difficult to compare man-hours, it seems that the results
%reported in \cite{larochelle2007empirical} required considerably more work
%than ours. In the website companion, the authors describe the sophisticated
%parameter search heuristics they used, which involved several stages of
%hyper-parameter tuning and iteratively refined grid searches, some of which
%were apparently performed manually\footnote{E.g. \emph{``As for the
%optimization
%    hyper-parameters, we would proceed by first trying a few
%    combinations of values for the stochastic gradient descent
%    learning rate of the supervised and unsupervised phases... The
%    first trials would simply give us a trend on the validation set
%    error for these parameters... and we would then consider that
%    information in selecting appropriate additional trials''.}}.
%Moreover, the authors state that at least for some of the deep learning
%methods used, a single training of the model required more than a day. In
%contrast, our parameter tuning was coarsely discretized, automated, and
%involved just a few parameters. We used a simple non-optimized MATLAB
%implementation, and for each width choice, building a family of networks for
%all $\Delta$ and $\gamma$ required at most a couple of hours using a single
%CPU.

We report the test error results (percentages of misclassified test examples)
in the table below. Each dataset number corresponds to the numbering of the
dataset descriptions above. For each dataset, we report the test error, and
in parenthesis indicate the depth/width of the network (where depth
corresponds to $\Delta$, so it includes the output layer). For comparison, we
also include the test error results reported in
\cite{larochelle2007empirical} for the other algorithms. Note that most of
the MNIST-related datasets correspond to multiclass classification with $10$
classes, so any result achieving less than 90\% error is non-trivial. \vskip
0.5cm
\begin{tabular}{c||c|c|c|c|c|c||c}
Dataset No.  & SVM-RBF & SVM-Poly & NNet & DBN-3 & SAA-3 & DBN-1 & \textbf{Basis Learner}\\\hline\hline
(\ref{d1})& 3.03 & 3.69 & 4.69 & 3.11 & 3.46 & 3.94 & 3.56 (5/600)\\\hline
(\ref{d2})& 11.11 & 15.42 & 18.11 & 10.30 & 10.30 & 14.69 & 10.30 (5/250)\\\hline
(\ref{d3})& 22.61 & 24.01 & 27.41 & 16.31 & 23.00 & 16.15& 22.43 (4/150)\\\hline
(\ref{d4})& 14.58 & 16.62 & 20.04 & 6.73 & 11.28 & 9.80& 9.17 (7/250) \\\hline
(\ref{d5})& 55.18 & 56.41 & 62.16 & 47.39 & 51.93 & 52.21& 50.47 (4/150)\\\hline
(\ref{d6})& 2.15 & 2.15 & 7.16 & 2.60 & 2.41 & 4.71 & 4.75 (4/50)\\\hline
(\ref{d7})&24.04 & 24.05 & 33.20 & 22.50 & 24.05 & 23.69 & 22.92 (3/100) \\\hline
(\ref{d8})& 19.13 & 19.82 & 32.25 & 18.63 & 18.41 & 19.92 & 15.45 (3/150)\\\hline
\end{tabular}
\vskip 0.5cm

From the results, we see that our algorithm performs quite well, building
deep networks of modest size which are competitive with (and for the
\texttt{Convex} dataset, even surpasses) the previous reported results. The
only exception is the \texttt{Rectangles} dataset (dataset no. \ref{d6}),
which is artificial and very small, and we found it hard to avoid overfitting
(the training error was zero, even after tuning $\lambda$). However, compared
to the other deep learning approaches, training our networks required minimal
human intervention and modest computational resources. The results are also
quite favorable compared to kernel predictors, but the predictors constructed
by our algorithm can be stored and evaluated much faster. Recall that a
kernel SVM generally requires time and memory proportional to the entire
training set in order to compute a single prediction at test time. In
contrast, the memory and time requirements of the predictors produced by our
algorithm are generally at least $1-2$ orders of magnitudes smaller.

\begin{figure}[t]
\begin{center}
\includegraphics[scale=0.5]{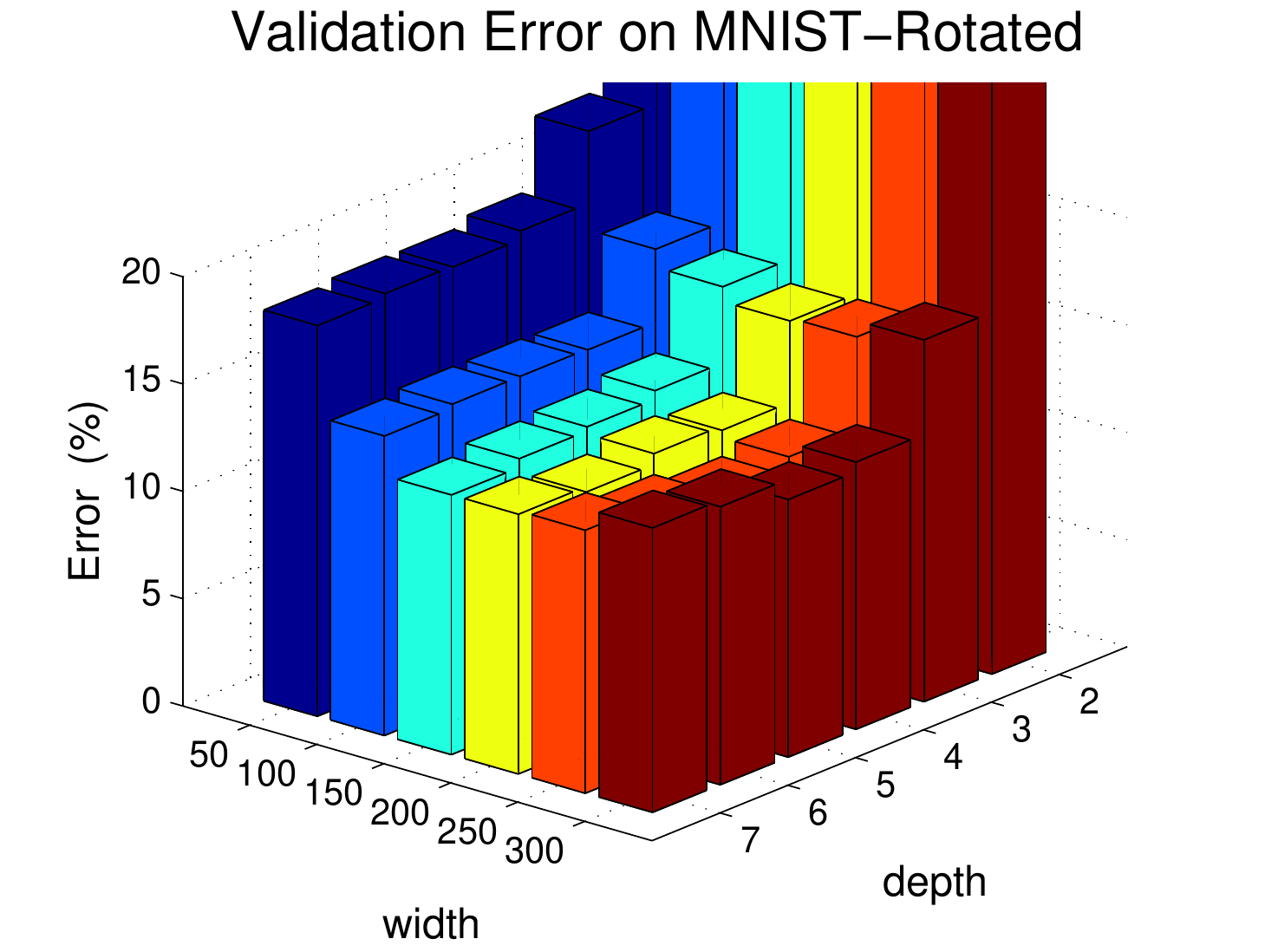}
\includegraphics[scale=0.5]{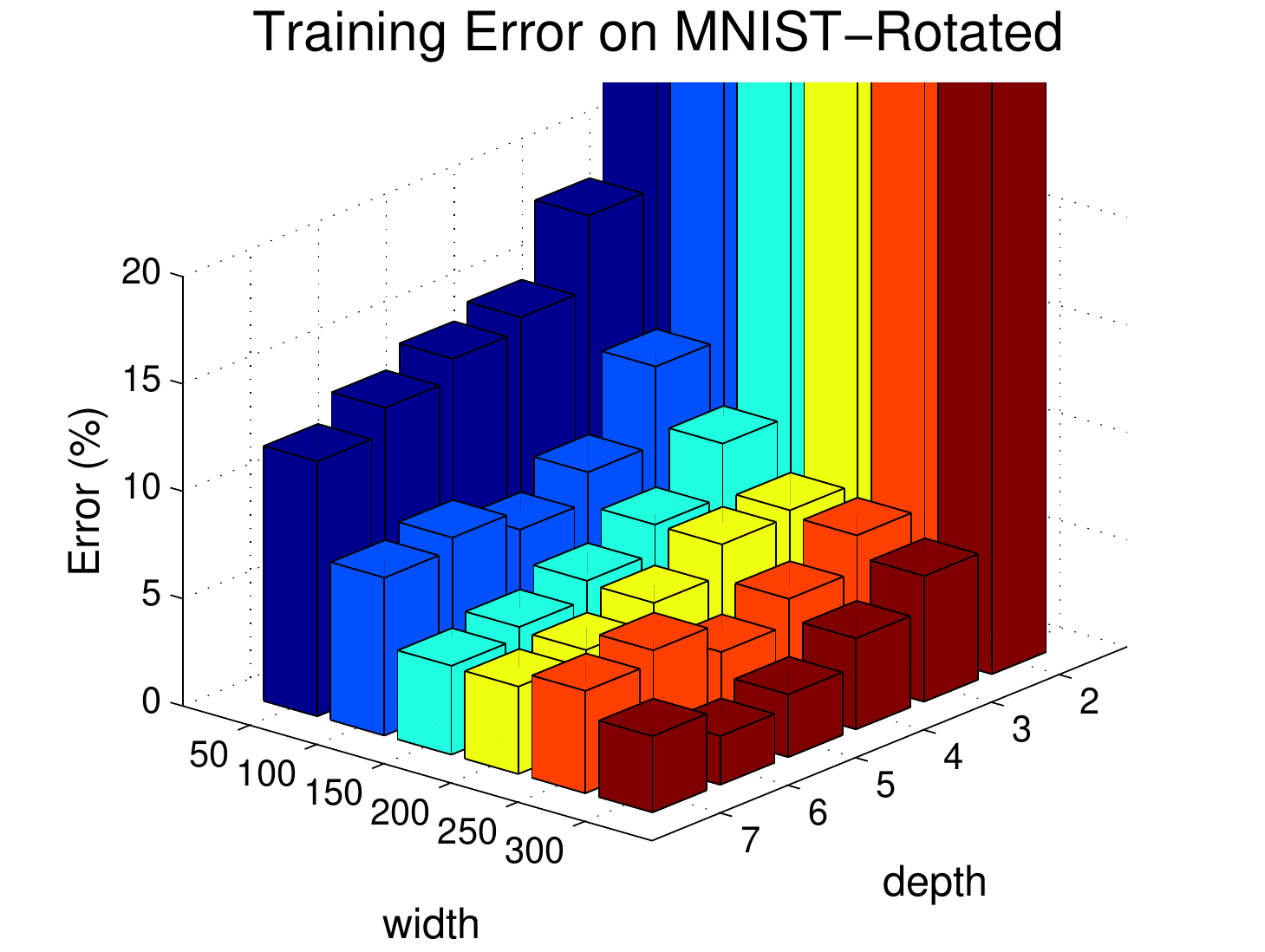}\\\vskip 0.5cm
\includegraphics[scale=0.5]{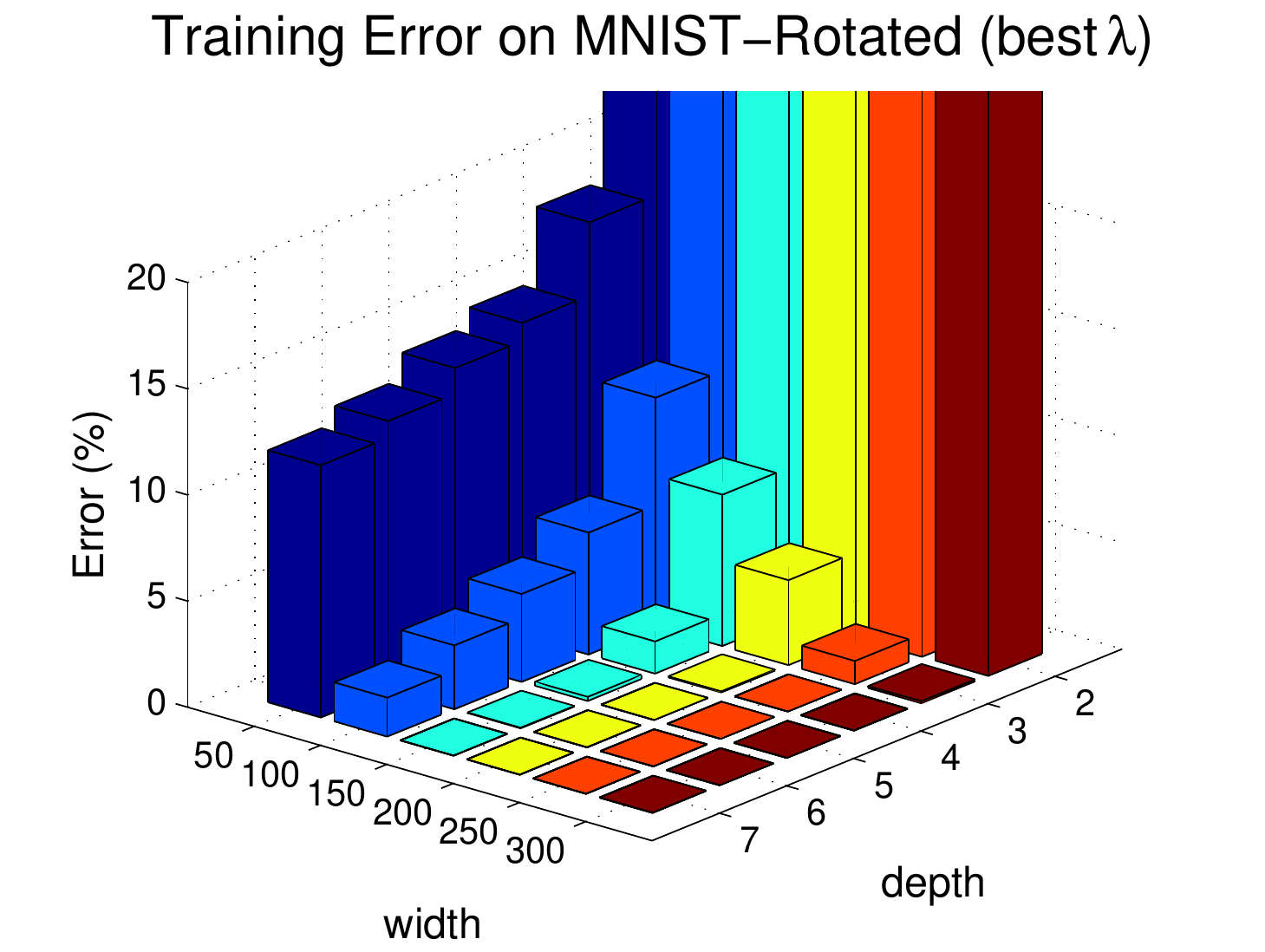}
\end{center}
\caption{Training and Validation Error Curves for the \texttt{MNIST-Rotated} dataset, as a function of trained network width and depth.}\label{fig:biasvariance}
\end{figure}

It is also illustrative to consider training/generalization error curves for
our algorithm, seeing how the bias/variance trade-off plays out for different
parameter choices. We present results for the \texttt{MNIST-rotated} dataset,
based on the data gathered in the parameter tuning stage (where the algorithm
was trained on the first 10,000 training examples, and tested on a validation
set of 2,000 examples). The results for the other datasets are qualitatively
similar. We investigate how 3 quantities behave as a function of the network
depth and width:
\begin{itemize}
\item The validation error (for the best choice of regularization parameter
    $\lambda$ in the output layer)
\item The corresponding training error (for the same choice of $\lambda$)
\item The lowest training error attained across all choices of $\lambda$
\end{itemize}
The first quantity shows how well we generalize as a function of the network
size, while the third quantity shows how expressive is our predictor class.
The second quantity is a hybrid, showing how expressive is our predictor
class when the output layer is regularized to avoid too much overfitting.

The behavior of these quantities is presented graphically in Figure
\ref{fig:biasvariance}. First of all, it's very clear that this dataset
requires a non-linear predictor: For a network depth of $2$, the resulting
predictor is just a linear classifier, whose train and test errors are around
50\% (off-the-charts). Dramatically better results are obtained with deeper
networks, which correspond to non-linear predictors. The lowest attainable
training error shrinks very quickly, attaining an error of virtually $0$ in
the larger depths/widths. This accords with our claim that the Basis Learner
algorithm is essentially a universal learning algorithm, able to
monotonically decrease the training error. A similar decreasing trend also
occurs in the training error once $\lambda$ is tuned based on the validation
set, but the effect of $\lambda$ is important here and the training errors
are not so small. In contrast, the validation error has a classical unimodal
behavior, where the error decreases initially, but as the network continues
to increase in size, overfitting starts to kick in.

Finally, we also performed some other experiments to test some of the
decisions we made in implementing the Basis Learner approach. In particular:
\begin{itemize} \item Choosing the intermediate layer's
  connections to be sparse (each node computes the product of only two
  other nodes) had a crucial effect. For example, we experimented with
  variants more similar in spirit to the VCA algorithm in \cite{VCA},
  where the columns of $F$ are forced to be orthogonal. This
  translates to adding a general linear transformation between each
  two layers. However, the variants we tried tended to perform worse,
  and suffer from overfitting. This may not be surprising, since these
  linear transformations add a large number of additional parameters,
  greatly increasing the complexity of the network and the risk of
  overfitting. \item Similarly, performing a linear transformation of
  the data in the first layer seems to be important. For example, we
  experimented with an alternative algorithm, which builds the first
  layer in the same way as the intermediate layers (using single
  products), and the results were quite inferior.
  While more experiments are required to explain this, we note that without this
  linear transformation in the first layer, the resulting predictor
  can only represent polynomials with a modest number of monomials
  (see Remark \ref{remark:architecture}). Moreover, the monomials tend to be
  very sparse on sparse data. \item As mentioned earlier, the
  algorithm still performed well when the exact SVD computation in the
  first layer construction was replaced by an approximate randomized
  SVD computation (as in \cite{halko2011finding}). This is useful in
  handling large datasets, where an exact SVD may be computationally expensive.
\end{itemize}

We end by emphasizing that these experimental results are preliminary, and
that much more work remains to further study the new learning approach that
we introduce here, both theoretically and experimentally.

\paragraph{Acknowledgements}
This research was funded in part by the Intel Collaborative Research
Institute for Computational Intelligence (ICRI-CI).

\bibliographystyle{plain}
\bibliography{DPN}
\ignore{
 \input{appendix}}
\end{document}